\documentclass[12pt]{article}
\usepackage{fancyhdr}
\usepackage{enumerate}
\usepackage{amssymb}
\usepackage{amsmath}
\usepackage{mathrsfs}
\usepackage{graphicx}
\usepackage{subfigure}
\usepackage{color}
\usepackage{tikz}
\usetikzlibrary{shapes,arrows,positioning,angles,quotes,calc}
\usepackage{amsthm}
\usepackage{algorithmic}
\usepackage{ragged2e}
\date{}

\usepackage{geometry}
\newtheorem{theorem}{Theorem}

\newtheorem{lemma}{Lemma}

\theoremstyle{remark}
\newtheorem{remark}{Remark}
\theoremstyle{definition}

\newtheorem{assumption}{Assumption}

\usepackage{authblk}
\usepackage{multirow}
\usepackage{mathtools}

\usepackage[ruled,vlined]{algorithm2e}
\fancypagestyle{firstpage}
{
\fancyhead[R]{}
\fancyhead[L]{This work has been submitted to the IEEE for possible publication. Copyright may be transferred without notice, after which this version may no longer be accessible.}
}
\begin{document}

\title{Convergent Reinforcement Learning Algorithms for Stochastic Shortest Path Problem}

\author[1]{Soumyajit Guin}
\author[1]{Shalabh Bhatnagar}
\affil[1]{Department of Computer Science and Automation, Indian Institute of Science, Bengaluru 560012, India\footnote{Email: gsoumyajit@iisc.ac.in, shalabh@iisc.ac.in. This work was supported in part by 
Project No.~DFTM/02/3125/M/04/AIR-04 from DRDO under DIA-RCOE and by the 
Walmart Centre for Tech Excellence, IISc.}}
\maketitle

\begin{abstract}
In this paper we propose two algorithms in the tabular setting and an algorithm for the function approximation setting for the Stochastic Shortest Path (SSP) problem. SSP problems form an important class of problems in Reinforcement Learning (RL), as other types of cost-criteria in RL can be formulated in the setting of SSP. We show asymptotic almost-sure convergence for all our algorithms. We observe superior performance of our tabular algorithms compared to other well-known convergent RL algorithms. We further observe reliable performance of our function approximation algorithm compared to other algorithms in the function approximation setting.
\end{abstract}

\section{Introduction}\label{intro}
\thispagestyle{firstpage}
\fancyhead[R]{This}
Reinforcement Learning (RL) is a machine learning paradigm where an agent takes actions in an uncertain environment. The goal of the agent is to maximize/minimize the expected rewards/costs it gets by taking those actions. See \cite{SuttonBarto2018}, \cite{NDPbook} for text book treatment of RL. In this work, we consider the Stochastic Shortest Path (SSP) problem (see \cite{NDPbook}, \cite{bertsekas2012dynamic}), that is typically characterized by the presence of a Goal or Terminal state. We propose three algorithms for the SSP problem. Two of these algorithms are proposed for the tabular setting. Here we maintain a table for all the states and a table for all the state-action pairs. Further, we propose an  algorithm for the function approximation setting. Here, instead of maintaining tables, we approximate the value and policy functions. The SSP is a powerful setting to solve because other types of cost criteria namely Finite Horizon, Infinite Horizon Discounted Cost and Long-run Average Cost, respectively, can all be equivalently formulated as suitable SSP problems  \cite{NDPbook,bertsekas1998}. It is also a useful problem because most practitioners use this as the objective, where a trajectory ends when some goal or terminal state is reached, instead of a single long trajectory that infinite horizon problems would typically result in \cite{li2023deep}.  

\subsection{Related Work}
To begin with, we briefly describe two assumptions commonly used for SSP problems.  If under a stationary policy the termination state is reached with probability $1$, it is called a proper policy. For a precise definition, see Section~\ref{prelim}. The assumptions are as follows:
\begin{itemize}
    \item Basic Assumption: Every policy is a proper policy.
    \item Standard Assumption: There exists at least one proper policy. Further, for policies that are not proper, every cycle must have a positive cost.
\end{itemize}

In \cite{tsitsiklis1994}, it is shown that Q-Learning converges under the Basic assumption, and also under the Standard assumption, provided the stability (boundedness) of the iterates is guaranteed. In \cite{yu2013boundedness}, it is shown that the iterates of Q-Learning do remain stable under the Standard assumption. In \cite{yu2013q}, the authors propose an  improved Q-learning algorithm for the SSP Problem under the Standard assumption, in which minimization at every step over all controls is not required. In \cite{teichteil2012}, a modified assumption to the SSP problem is proposed, where the probability to reach the goal state under every policy is less than $1$. But the proposed algorithm is model-based (where the probability and cost function are assumed known). This is unlike our algorithms that are data-driven and model-free as we do not assume knowledge of the system model. In \cite{patek1999}, numerical procedures such as Value Iteration (VI) and Policy Iteration (PI) are investigated under the Standard assumption for zero-sum games, where each player has the SSP problem as its objective. In \cite{guillot2020}, it is shown that model based algorithms like VI and PI also converge under the more general assumption, namely that improper policies can have cycles with nonnegative cost. In \cite{neu2012}, certain regret bounds for SSPs are obtained but under the restrictive assumption where the MDP does not have cycles. There is a whole line of work on model-based RL for the SSP problem, which we will not dive into, see \cite{chen2022}, \cite{min2022}, \cite{yin2022offline}. These methods are computationally heavy \cite{plaat2023high}. In \cite{randour2015variations}, some variations to the standard SSP problem are provided under the Basic assumption, but in our paper we solve the standard SSP problem, which is just to minimize the expected value. In \cite{kwon2011}, an enhanced Q-Learning algorithm is proposed in the tabular setting for the SSP problem under the Basic assumption.

\begin{table*}
\centering
\begin{tabular}{|p{1.6cm}|p{2cm}|p{2cm}|p{2cm}|p{2cm}|p{2cm}|}
\hline
Assump-tion & Dynamic Programming (VI, PI, etc.)& Tabular Q-Learning & Tabular Actor-critic and Critic-Actor & SARSA with LFA & Function Approximtion Actor-critic\\
\hline
Basic & \cite{NDPbook} & \cite{tsitsiklis1994} (Stability and Convergence) &Our Paper (Stability and  Convergence)&\cite{gordon2000} (not a rigorous proof)& Our Paper (Stability and Convergence)\\
\hline
Standard & \cite{bertsekas2012dynamic} & \cite{tsitsiklis1994} (Convergence assuming Stability), \cite{yu2013boundedness} (Stability) &Can be extended from our paper & NA& Can be extended from our paper\\
\hline
\end{tabular}
\caption{Comparison with Relevant Literature}
\label{table1}
\end{table*}

 The Q-learning and SARSA algorithms have nice convergence guarantees in the tabular setting \cite{tsitsiklis1994}, \cite{singh2000convergence}. In this paper we propose two convergent algorithms for the SSP problem in the tabular setting. The first of these is the actor-critic (AC) algorithm, where the value function parameter corresponding to the critic is updated along the faster timescale while the policy parameter or the actor is updated along the slower timescale. The second tabular algorithm that we present is the critic-actor (CA) algorithm, wherein the timescales are reversed. Whereas the AC algorithm mimics policy iteration, CA is seen to mimic value iteration, so the underlying philosophy in these algorithms is different. Nonetheless, we observe that  there is not much difference between the convergence guarantees of these two algorithms. Both algorithms exhibit stable learning due to their on-policy nature. So its a personal preference for the user what algorithm to use in the tabular setting. In Section~\ref{numerical}, we observe that our tabular algorithms AC and CA can perform better than other tabular algorithms such as  Q-Learning and SARSA when theoretically backed exploration and exploitation is used.

When it comes to the function approximation setting, Q-Learning is not function approximation friendly \cite{baird1995} due to its off-policy nature \cite{tsitsiklis1997}. Asymptotic almost-sure convergence of Q-Learning with linear-function approximation (LFA) has been analysed in \cite{melo2008}, but the conditions there are hard to satisfy without an on-policy algorithm (see Theorem 1 of \cite{melo2008}). This brings us to the asymptotic almost-sure convergence analysis of SARSA with LFA done in \cite{melo2008}. However, in their analysis of SARSA with LFA, the authors need the behavioral policy to be Lipschitz continuous, with a small Lipschitz constant (see Theorem 2 of \cite{melo2008}). This is not possible without a parameterized policy like softmax. While their work is based on the Discounted Cost problem, similar conclusions can be drawn for the SSP problem as well. Convergence of SARSA with LFA in the expected sense has recently been analyed in \cite{zhang2023sarsa} for the Discounted Cost problem, which addresses the issue of \cite{melo2008}. \cite{zhang2023sarsa} shows that the softmax policy with high temperature can act as a policy with small Lipschitz constant. In the experiments,  \cite{zhang2023sarsa} provides an SSP example that we will use in Section~\ref{numerical}. In \cite{gordon2000}, almost sure convergence of SARSA with LFA under the $\epsilon$-greedy policy for the SSP problem is shown, but they do not discuss about the optimality of the final Q-values obtained. A concise literature review for works on the SSP prroblem is given in Table~\ref{table1}.

It was shown in \cite{CriticActor2023} that if we reverse the timescales of traditional actor-critic algorithms for the Discounted cost problem in the tabular setting, we still get a valid algorithm. This algorithm has been referred to as the Critic-Actor algorithm and it has been shown to converge to the same optimal point as Actor-Critic. In this paper, we propose both Actor-Critic and Critic-Actor algorithms for the SSP problem in the tabular setting. We further propose an Actor-Critic algorithm in the function approximation setting for the SSP problem.

\subsection{Contributions}
\begin{enumerate}
\item In this work, we propose two separate two-timescale algorithms in the tabular setting for the SSP problem.

\item We further propose a two-timescale algorithm for the SSP problem in the function approximation setting.

\item We show the asymptotic almost-sure convergence of all our algorithms.

\item We show the results of experiments that demonstrate that our tabular algorithms perform better than other popular algorithms.

\item In our experiments, we further observe that our function approximation algorithm converges and performs reliably in contrast to other algorithms in the function approximation setting.
\end{enumerate}

\section{Preliminaries}\label{prelim}
We consider a Markov Decision Process(MDP) that has a finite state space $S$, finite action space $A$. For simplicity, we assume every action is feasible in every state. Let $\{X_n\}$ be the underlying Markov Chain which takes values in the state space S. The MDP satisfies the property:
\[
\forall i \in S,P(X_{n+1}=i|X_m,U_m,m\leq n)=p(X_n,U_n,i)\text{ a.s.}
\]
Here $U_n$ is the action chosen at time instant $n$ when current state is $X_n$.The transition function $p:S\times A \times S\rightarrow [0,1]$ is such that $p(i,u,j)$ is the probability of transitioning to state $j$ from state $i$ when action chosen is $u$. Let $g:S\times A \times S \rightarrow \mathbb{R}$ be the single stage cost. Let $i_0 \in S$ be a zero cost absorbing state or the terminating state. Once the system enters the state, it remains in the state with cost $0$, which means,
\[
p(i_0,u,i_0)=1, g(i_0,u,i_0)=0, \forall u \in A.
\]
Let us define the state space excluding the terminal state as $S^-:=S\backslash \{i_0\}$. In this paper we consider Stationary Randomized Policies (SRP) which means a policy $\pi:S^-\times A \rightarrow [0,1]$ is such that $\pi(i,u)$ is the probability of taking action $u$ from state $i$. Thus, for any $i\in S^{-}$, $\sum_{u\in A} \pi(i,u)=1$. 
Note that no action needs to be taken once the system reaches state $i_0$. Let action $U_n$ be sampled from $\pi(X_n,.)$.

Our goal is to find a SRP $\pi$ which minimizes the expected total cost:
\begin{equation}\label{avg}
J(\pi)=\mathbb{E}\Big[\sum_{t=0}^{\tau(\pi)-1}g(X_t,U_t,X_{t+1})\Big],
\end{equation}
where $\tau(\pi)$ is the first reaching time to state $i_0$ under policy $\pi$. Let $|S|$ and $|S^-|$ respectively denote the cardinalities of the sets $S$ and $S^-$. 
\begin{assumption}\label{a1}
Every SRP is proper. In other words, there is a positive probability that $i_0$ will be reached  in at most $|S^-|$ time steps under every SRP $\pi$ regardless of the initial state. Formally stated, under every SRP $\pi$,
\begin{equation}\label{reach}
\max_{i \in S^-}P(X_{|S^-|}\neq i_0|X_0=i,\pi)<1.
\end{equation}
\end{assumption}
The value function $V^\pi(i),i \in S^-$ under policy $\pi$ is defined as:
\[
V^\pi(i) := E\left[ \sum_{t=0}^{\tau(\pi)-1} g(X_t, U_t, X_{t+1})\mid X_0=i\right].
\]
Then $V^\pi(i),i \in S^-$ is the solution to the Bellman equation (for the given policy $\pi$):
\begin{equation}\label{ssp-bepi}
\begin{split}
&V^\pi(i)=\sum_{u \in A}\pi(i,u)\Big(\sum_{j \in S}p(i,u,j)g(i,u,j)\\
&+\sum_{j\in S^-} p(i,u,j)V^\pi(j)\Big).
\end{split}
\end{equation}
The Q-value function $Q^\pi(i,u),i\in S^-,u \in A$ under policy $\pi$ is defined as:
\begin{equation}\label{Q-def}
Q^\pi(i,u)=\sum_{j \in S}p(i,u,j)g(i,u,j)+\sum_{j\in S^-} p(i,u,j)V^\pi(j).
\end{equation}
Now let 
\[
V^*(i)=\min_\pi V^\pi(i),i \in S^-.
\]
Then $V^*(i),i\in S^-$ is the solution to the Bellman equation:
\begin{equation}
\begin{split}
&V^*(i)=\min_{u \in A}\Big(\sum_{j \in S}p(i,u,j)g(i,u,j)\\
&+\sum_{j\in S^-} p(i,u,j)V^*(j)\Big).
\end{split}
\end{equation}

\section{Actor-Critic and Critic-Actor Algorithms in the Tabular Setting}
\begin{algorithm}[tb]\label{algorithm1}
\caption{Actor-Critic and Critic-Actor Tabular}
\textbf{Input}: Step-sizes $a(n), b(n)$, and  $Y_n,Z_n$ satisfying Assumption~\ref{a3}.
\begin{algorithmic}[1]
\STATE Initialize $V_0(i),i\in S^-$; $\theta_0(i,u),i\in S^-,u \in A$.
\FOR{each timestep \(n \geq 0 \)}
\STATE Play action $\chi_n(i)$ from state $i$ and collect next state $\eta^1_n(i,\chi_n(i))$ and cost $g(i,\chi_n(i),\eta^1_n(i,\chi_n(i)))$, where $i=Y_n$.
\STATE Perform critic update \eqref{critic2}.
\STATE Take action $u$ in state $i$ and collect next state $\eta^2_n(i,u)$ and cost $g(i,\chi_n(i),\eta^1_n(i,u))$, where $(i,u)=Z_n$.
\STATE Perform actor update \eqref{actor2}.
\ENDFOR
\end{algorithmic}
\end{algorithm}

We make the following assumption on the policy structure.
\begin{assumption}\label{a2}
    We assume that $\pi$ is the soft-max policy parameterized by $\theta$:
    \[
    \pi_\theta(i,u)=\frac{\exp(\theta(i,u))}{\displaystyle\sum_{u' \in A} \exp(\theta(i,u'))},i \in S^-, u \in A.
    \]
\end{assumption}
We will use $\pi$ and $\theta$ interchangeably henceforth and 
present a two timescale off-line algorithm. Off-line here means the states and state-action pairs to be updated are sampled from some distribution. This is a more general case of the more widely studied on-line algorithms \cite{Konda2003,bhatnagar2009}, where the distribution is the estimated policy itself. We lay out the conditions required on the step-sizes of our algorithm. These conditions are generalized from \cite{borkar1998} to the multi-timescale case (see \cite{konda1999},\cite{CriticActor2023}).

Let $\{Y_n\},\{Z_n\}$ be $S^-$ valued and $S^-\times A$ valued processes respectively. Then $Y_n=i$ means that the $i$-th component will be updated and $Z_n=(i,u)$ means that the $(i,u)$-th component will be updated at time instant $n$. Define $\nu_1(i,n),\nu_2(i,u,n)$ as follows: For $n\geq 1$, 
\[
\begin{split}
&\nu_1(i,n)=\sum_{m=0}^{n-1}I\{Y_m=i\} \text{ with }\nu_1(i,0)=0\\
&\nu_2(i,u,n)=\sum_{m=0}^{n-1}I\{Z_m=(i,u)\} \text{ with }\nu_2(i,u,0)=0.
\end{split}
\]
Here $I\{.\}$ is the indicator function, $\nu_1(i,n)$ is the number of times the $i$-th component is updated and $\nu_2(i,u,n)$ is the number of times the $(i,u)$-th component is updated until time instant $n$.

We make the following assumptions on the step-sizes:
\begin{assumption}[Step-Sizes]
\label{a3}
$a(n),b(n),n\geq 0$ are positive sequences satisfying the following conditions. Let $d(n),n\geq 0$ represent $a(n),b(n),n\geq 0$. Then
\begin{itemize}
\item[(i)] ${\displaystyle \sum_n d(n)=\infty}$, ${\displaystyle \sum_n d(n)^2<\infty}$.
\end{itemize}
This is the standard Robbins-Monro condition on step-sizes for Stochastic Approximation. We additionally require the following:
\begin{itemize}
\item[(ii)] For any $x\in (0,1)$, ${\displaystyle \sup_n \frac{d([xn])}{d(n)}<\infty},$ where $[xn]$ denotes the integer part of $xn$.
\item[(iii)] For any $x\in (0,1)$, ${\displaystyle D(n) := \sum_{i=0}^{n} d(i)}$, $n\geq 0$, we have
	   ${\displaystyle \frac{D([yn])}{D(n)} \rightarrow 1}$, as $n\rightarrow\infty$, 
	   uniformly over $y\in [x,1]$.
\end{itemize}
Conditions (ii) and (iii) are technical assumptions which are required for asynchronous type updates, see (2.2) and (2.3) of \cite{borkar1998}. For $x>0$, define 
\begin{eqnarray*}
N_1(n,x) &=& \min\{ m>n \mid \sum_{k=n+1}^{m} \bar{a}(k) \geq x\},\\
N_2(n,x) &=& \min\{ m>n \mid \sum_{k=n+1}^{m} \bar{b}(k) \geq x\},
\end{eqnarray*}
where $\bar{a}(n) \equiv a(\nu_1(Y_n,n))$ and $\bar{b}(n) \equiv b(\nu_2(Z_n,n))$, respectively.
\begin{itemize}
\item[(iv)] For $x>0$, the limits
\[\lim_{n\rightarrow\infty} \frac{\sum_{j=\nu_1(i,n)}^{\nu_1(i,N_1(n,x))} a(j)}{\sum_{j=\nu_1(k,n)}^{\nu_1(k,N_1(n,x))} a(j)}, \mbox{   }
\lim_{n\rightarrow\infty} \frac{\sum_{j=\nu_2(i,u,n)}^{\nu_2(i,u,N_2(n,x))} b(j)}{\sum_{j=\nu_2(k,v,n)}^{\nu_2(k,v,N_2(n,x))} b(j)},
\]
exist almost surely for states $i,k\in S^-$ and state-action tuples $(i,u),(k,v)\in S^-\times A$ in the two limits respectively.	
\end{itemize}
Intuitively this means the updating of different components is evenly spread.
\begin{itemize}
\item[(v)] For Actor-Critic: $\frac{b(n)}{a(n)}\rightarrow 0$. For Critic-Actor: $\frac{a(n)}{b(n)}\rightarrow 0$.
\end{itemize}
This is the two timescale stochastic approximation condition, see chapter 8 of \cite{borkar2022}.
\begin{itemize}
\item[(vi)] There exists a constant $\kappa>0$ such that the following holds almost surely $\forall i \in S^-, u\in A$:
\[
\liminf_{n \to \infty}\frac{\nu_1(i,n)}{n}\geq \kappa,\liminf_{n \to \infty}\frac{\nu_2(i,u,n)}{n}\geq \kappa .
\]
\end{itemize}
It means that the updates for all states and state-action pairs are performed infinitely often.
\end{assumption}
    
\begin{remark}\label{rem1}
 Example of step sizes that satisfy these conditions for Actor-Critic: $a(n)=\log(n+2)/(n+2),b(n)=1/(n+1)$; for Critic-Actor: $a(n)=1/((n+2)\log(n+2)), b(n)=\log(n+2)/(n+2)$. To increase or decrease the step sizes we can multiply any positive constant to them which will also satisfy these conditions.
\end{remark}

Let $V_n(i),i \in S^-$, denote our critic updates with fixed $V_n(i_0)=0, \forall n$. Let $\theta_n(i,u),i \in S^-, u \in A$, denote our actor updates. Let $\{\eta^1_n(i,u)\},\{\eta^2_n(i,u)\},i \in S^-, u \in A$ denote families of $S$-valued independent and identically distributed (i.i.d) random variables with the law $p(i,u,\cdot)$. Let $\chi_n(i),i \in S^-$ be $A$-valued random variables with the law $\pi_{\theta_n}(i,\cdot)$. Let $\Gamma_{\theta^0}(\cdot),\theta^0>0$ be a projection operator to the set $[-\theta^0,\theta^0]$. The updates of our algorithm are given by: $\forall i\in S^-,u \in A$,
\begin{equation}\label{critic2}
\begin{split}
&V_{n+1}(i)=V_n(i)+a(\nu_1(i,n))\big[g(i,\chi_n(i),\\
&\eta^1_n(i,\chi_n(i)))+V_n(\eta^1_n(i,\chi_n(i)))-V_n(i)\big]I\{Y_n=i\},
\end{split} 
\end{equation}
\begin{equation}\label{actor2}
\begin{split}
&\theta_{n+1}(i,u)=\Gamma_{\theta^0}\big(\theta_n(i,u)-b(\nu_2(i,u,n))
\big[g(i,u,\\
&\eta^2_n(i,u))+V_n(\eta^2_n(i,u))-V_n(i)\big]I\{Z_n=(i,u)\}\big).\
\end{split}
\end{equation}
\begin{remark}
The recursions of both Actor-Critic and Critic-Actor are the same, only the difference is in the step size condition~\ref{a3}(v). Further, for a maximization (instead of a minimization) problem, one needs to change the $-$ sign to $+$ sign after $\theta_n(i,u)$ in recursion \eqref{actor2} so that the $\theta$-recursion then performs an `ascent' instead of a `descent'.
\end{remark}
The algorithm is formally given in Algorithm 1. Our Actor-Critic and Critic-Actor algorithms are essentially the same as the Actor-Critic and Critic-Actor algorithms for the discounted cost problem given in \cite{konda1999} and \cite{CriticActor2023} respectively, barring that we take discount factor as $1$. The novelty lies in the fact that our algorithms are for the SSP setting and not discounted cost. Further, our Convergence Analysis given in Section~\ref{convergence-tabular}, is quite different from the Convergence Analysis done in \cite{konda1999} and \cite{CriticActor2023}.

\subsection*{Online Tabular Algorithms}
Our off-line algorithms can easily be converted to online algorithms in the following manner. Let $\{\eta^1_n(i,u)\}=\{\eta^2_n(i,u)\}, \forall i \in S^-, u \in A$. An episode comprises of a sequence of the four tuples of state, action, reward and next state, starting from the start state and ending in the terminal state $i_0$. For the $m$th episode, with $m\geq 0$, let the $k$th state be denoted $i_k(m)$, the $k$th action taken be denoted $u_k(m)$ and the length of the episode be denoted $\tau'(m)$. Let $n(m,k)=\sum_{t=0}^{m-1}\tau'(t)+k$ with $n(0,k)=k$. The action $u_k(m)=\chi_{n(m,k)}(i_k(m))$. Now let $Y_{n(m,k)}=i_k(m)$ and $Z_{n(m,k)}=(i_k(m),u_k(m))$. This will ensure that only those components of the states and state-action tuples that are visited get updated at each episode. The updates for online tabular algorithms are given below: $\forall i\in S^-,u \in A$,

\begin{equation}\label{critic3}
\begin{split}
&V_{n(m,k)+1}(i)=V_{n(m,k)}(i)+a(\nu_1(i,n))\big[g(i,u_k(m),\\
&i_{k+1}(m))+V_{n(m,k)}(i_{k+1}(m))-V_{n(m,k)}(i)\big]\\
&I\{i_k(m)=i\},
\end{split} 
\end{equation}
\begin{equation}\label{actor3}
\begin{split}
&\theta_{n(m,k)+1}(i,u)=\Gamma_{\theta^0}\big(\theta_{n(m,k)}(i,u)-b(\nu_2(i,u,n))\\
&\big[g(i,u,i_{k+1}(m))+V_{n(m,k)}(i_{k+1}(m))-V_{n(m,k)}(i)\big]\\
&I\{(i_k(m),u_k(m))=(i,u)\}\big).\
\end{split}
\end{equation}

\section{Convergence Analysis of Algorithms in Tabular Setting}\label{convergence-tabular}
We proceed to prove the convergence of both Actor-Critic (AC) and Critic-Actor (CA) algorithm in the tabular setting. Section 4 of \cite{konda1999} gives the convergence of Asynchronous Stochastic Approximation with two timescales. We will use those results to prove the convergence of AC and some extra results from \cite{benaim2005} to prove the convergence of CA. (A1) and (A6) of \cite{konda1999} are built in Assumption~\ref{a3} except that $a(n)=o(b(n))$ for CA. (A5) of \cite{konda1999} is irrelevant for us, as our updates do not involve delays. Let 
\[
\begin{split}
&\mathcal{F}_n \stackrel{\triangle}{=} \sigma(\chi_m(i),\eta^1_m(i,u),\eta^2_m(i,u),m<n;V_m(i),\theta_m(i,u),\\
&Y_m,Z_m,m\leq n,i \in S^-, u \in A), n \geq 0,
\end{split}
\] 
denote an increasing family of associated sigma fields. This keeps a record of the history of the algorithm till time instant $n$. Let us consider an auxiliary MDP with probability function $p$ and reward function $g'(i,u,j)=1, \forall i \in S^-,u \in A,j \in S$. Let $\xi(i),i \in S^-$ be the solution to the Bellman equation for the auxiliary MDP:
\begin{equation}
\begin{split}
\xi(i)=\max_{u \in A}\Big[1+\sum_{j\in S^-}p(i,u,j)\xi(j)\Big],
\end{split}
\end{equation}
with $\xi(i_0)=0$. This will be required in Lemmas~\ref{stability-lemma} and \ref{lem-main}. We define the following vector variables:
\begin{equation}\label{V-def}
V:=(V(i),i\in S^-)^T,
\end{equation}
\begin{equation}\label{theta-def}
\theta=(\theta(i,u),i\in S^-,u \in A)^T.
\end{equation}
This will be used throughout our analysis. The Bellman operator $T_\pi$ under policy $\pi$, and the Bellman operator for optimality $T$ are defined as:
\begin{equation}
\begin{split}
&(T_\pi V)(i)=\sum_{u \in A}\pi(i,u)\Big(\sum_{j \in S}p(i,a,j)g(i,a,j)\\
&+\sum_{j\in S^-} p(i,u,j)V(j)\Big),i \in S,
\end{split}
\end{equation}
\begin{equation}
\begin{split}
&(T V)(i)=\min_{u \in A}\Big(\sum_{j \in S}p(i,a,j)g(i,a,j)\\
&+\sum_{j\in S^-} p(i,u,j)V(j)\Big),i \in S.
\end{split}
\end{equation}
We have the following stability lemma.
\begin{lemma}\label{stability-lemma}
$\sup_n ||V_n||<\infty$ a.s
\end{lemma}
\begin{proof}
We use Theorem 1 of \cite{tsitsiklis1994} to prove this. Our $V_n$ update \eqref{critic2} can be written as: For $i \in S^-$,
\[
\begin{split}
V_{n+1}(i)=V_n(i)+a(i,n)\big[F^n_i(V_n)-V_n(i)+w_i(n)\big],
\end{split}
\]
with $V_n(i_0)=0,\forall n$. Here $a(i,n)=a(\nu_1(i,n))I\{Y_n=i\}$. $F^n_i(V)$ for $V$ as in \eqref{V-def} is defined as:
\[
\begin{split}
&F_i^n(V)=\sum_{u \in A}\pi_{\theta_n}(i,u)\Big[\sum_{j \in S} p(i,u,j)g(i,u,j)\\
&+\sum_{j\in S^-} p(i,u,j)V(j)\Big].
\end{split}
\]
The noise term $w_i(n)$ is defined as:
\[
\begin{split}
&w_i(n)=g(i,\chi_n(i),\eta^1_n(i,\chi_n(i)))
+V_n(\eta^1_n(i,\chi_n(i)))\\
&-E[g(i,\chi_n(i),\eta^1_n(i,\chi_n(i)))+V_n(\eta^1_n(i,\chi_n(i)))|\mathcal{F}_n].
\end{split}
\]
Here $F^n_i$ is used instead of $F_i$ but it does not affect the conclusions.
By Assumption~\ref{a3} for all $i$, $\sum_na(i,n)=\infty$ a.s, $\sum_na(i,n)^2\leq \bar{C}$ for some $\bar{C}>0$. It is easy to see that $E[w_i(n)|\mathcal{F}_n]=0$ and:
\[
E[w_i(n)^2|\mathcal{F}_n]\leq A+B\max_j\max_{k\leq n}|V_k(j)|^2 \text{ a.s.}
\]
for some $A,B>0$. Let $F^n(V)=(F^n_i(V),i \in S^-)^T$. It can be verified that 
\[
\begin{split}
||F^n(V)||_\xi\leq D+\beta||V||_\xi,
\end{split}
\]
where,
\begin{equation}
\beta=\max_{i\in S^-}\frac{\xi(i)-1}{\xi(i)} \text{ and } ||V||_\xi=\max_{i\in S^-}\frac{|V(i)|}{\xi(i)}.
\end{equation}
The claim follows by Theorem 1 of \cite{tsitsiklis1994}.
\end{proof}
By Lemma~\ref{stability-lemma}, (A2) of \cite{konda1999} is verified.
For $(i,u)\in S^-\times A$, and for $V$ as in \eqref{V-def} let:
\[
k_{iu}(V) = \sum_{j \in S} p(i,a,j)g(i,a,j) +\sum_{j\in S^-} p(i,u,j)V(j)- V(i).
\]
Let the directional derivative of $\Gamma_{\theta^0}$ be defined as:
\begin{equation}\label{dir-def}
\bar{\Gamma}_{\theta^0}(x;y) = \lim_{\delta\downarrow 0} \left(\frac{\Gamma_{\theta^0}(x + \delta y) - x}{\delta}\right).
\end{equation}
For both AC and CA recursions, \eqref{critic2} can be written as follows: For $i \in S^-$,
\[
\begin{split}
&V_{n+1}(i) = V_n(i) + \bar{a}(n)I\{Y_n=i\}\big[\sum_{u \in A}\pi_{\theta_n}(i,u)\\
&k_{iu}(V_n)+N_{n+1}(i)\big],
\end{split}
\]
with $V_n(i_0)=0, \forall n$. For both AC and CA recursions, \eqref{actor2} can be written as: For $i \in S^-,u\in A$,
\[
\begin{split}
&\theta_{n+1}(i,u) =\Gamma_{\theta^0}\big(\theta_n(i,u) + \bar{b}(n) I\{Z_n=(i,u)\}\\
&\big[-k_{ia}(V_n)+ M_{n+1}(i,u)\big]\big).
\end{split}
\]
For $n\geq 0$, define the $\{\mathcal{F}_n\}$-adapted sequences:
\[
\begin{split}
&N_{n+1}(i) = g(i,\phi_n(i),\eta^1_n(i,\chi_n(i)))
+V_n(\eta^1_n(i,\chi_n(i)))\\
&-V_n(i)-\sum_{u \in A} \pi_{\theta_n}(i,u)k_{iu}(V_n),
\end{split}
\]
\[
\begin{split}
&M_{n+1}(i,u) = V_n(i)-g(i,u,\eta^2_n(i,u))-V_n(\eta^2_n(i,u))\\
&+k_{iu}(V_n).
\end{split}
\]
They are martingale difference sequences. We have the following lemma.
\begin{lemma}\label{noise-lemma}
For some $C_1,C_2>0$, 
\[||N_{n+1}||\leq C_1(1+||V_n||)\text{ a.s},\]
\[||M_{n+1}||\leq C_2(1+||V_n||)\text{ a.s},\]
where $N_{n+1}=[N_{n+1}(i),i \in S^-]^T$ and $M_{n+1}=[M_{n+1}(i,u),i \in S^-,u \in A]^T$.
\end{lemma}
\begin{proof}
Easy to see from the definition of $N_{n+1}$ and $M_{n+1}$.
\end{proof}
By Lemma~\ref{noise-lemma} and Lemma~\ref{stability-lemma}, (A4) of \cite{konda1999} is verified. This basically means our noise terms will converge and now we can deal with the limiting Ordinary Differential Equations (ODEs) and Differential Inclusions (DIs). 

\subsection{Convergence of Actor Critic}
Since our actor is updated along the slower timescale, we let $\theta_n \approx \theta$ for the analysis of the faster recursion with $\theta$ as in \eqref{theta-def}. 

Let $R_\theta$ be the vector with elements: 
\begin{equation}\label{R-def}
\begin{split}
R_\theta(i)=\sum_{u \in A} \pi_\theta(i,u)\sum_{j\in S}p(i,u,j)g(i,u,j),i \in S^-.
\end{split}
\end{equation}
Let $P_\theta$ be the matrix with elements: 
\begin{equation}\label{P-def}
P_\theta(i,j)=\sum_{u \in A} \pi_\theta(i,u)p(i,u,j),i,j \in S^-.
\end{equation}
 Consider now the ODE:
\begin{equation}
	\label{ode-c}
	\dot{V} = R_\theta +P_\theta V-V.
\end{equation}
We have the following:
\begin{lemma}
\label{lem1}
$V^\theta=(I-P_\theta)^{-1}R_\theta$ is the unique globally asymptotically stable equilibrium of \eqref{ode-c}. Also $V^\theta$ is Lipschitz continuous function of $\theta$.
\end{lemma}
\begin{proof}
For a matrix $A$ with elements $\{a_{ij}\}$, let one of the eigenvalues be $\lambda_A$. Then by Gershgorin Circle Theorem, $|\lambda_A|\leq\max_i \sum_j |a_{ij}|$. Let one of the eigenvalues of $P_\theta$ be $\lambda_P$. Now by \eqref{reach},  $P_\theta^{|S^-|}$ has eigenvalues $\lambda_P^{|S^-|}$ with $|\lambda_P^{|S^-|}
|<1$. Hence $|\lambda_P|<1$. Hence the matrix $P_\theta-I$ has all eigenvalues with negative real parts. The first claim follows.

Using Cramer's rule one can verify that $V^\theta$ is continuously differentiable in $\theta$. As $\theta$ lies in a compact set, $V^\theta$ is Lipschitz continuous.
\end{proof}

\begin{remark}
Let us define the function $f(\theta,V)$, which is the mean field of \eqref{critic2}:
\[
f(\theta,V)=[\sum_{u \in A}\pi_{\theta}(i,u)k_{iu}(V),i\in S^-]^T.
\]
It can be seen that $f(\theta,V)$ is not a Lipschitz continuous function but a Locally Lipschitz function for $||V||\leq B',B'>0$. But since by Lemma~\ref{lem1}, the ODE \eqref{ode-c} has a globally asymptotically stable equilibrium, the ODE takes values in a bounded set and the arguments of Section 4 of \cite{konda1999} still hold.
\end{remark}

The policy update \eqref{actor2} is on the slower timescale, hence it sees the critic update \eqref{critic2} as quasi-equilibrated. Then the policy update \eqref{actor2} can be written as: For $i \in S^-,u \in A$,
\begin{equation}\label{equilibriated-actor2}
\begin{split}
&\theta_{n+1}(i,u)=\Gamma_{\theta^0}\big(\theta_n(i,u)+\bar{b}(n)I\{Z_n=(i,u)\}\\
&\big[-k_{iu}(V^{\theta_n})+M_{n+1}(i,u)+\varepsilon'_n(i,u)\big]\big).
\end{split}
\end{equation}
By Lemma~\ref{lem1}, $\varepsilon'_n(i,u)\rightarrow 0$.
Let us consider the ODE: For $i \in S^-,u \in A$,
\begin{equation}\label{actor_ode}
\dot{\theta}(i,u)=\bar{\Gamma}_{\theta^0}(\theta(i,u);k_{iu}(V^\theta)).
\end{equation}
See \eqref{dir-def} for the definition of $\bar{\Gamma}_{\theta^0}(.;.)$. Let,
\[
\gamma^1_{iu}(\theta) :=\left\{
\begin{array}{ll}
0 & \mbox{ if } \theta(i,u) = \theta^0 \mbox{ and } k_{iu}(V^\theta) \leq 0,\\
   & \mbox{ or } \theta(i,u) = -\theta^0 \mbox{ and } k_{iu}(V^\theta) \geq 0,\\
1  & \mbox { otherwise.}
\end{array}
\right.
\]
Thus, the ODE \eqref{actor_ode} takes the form:
\begin{equation}
	\dot{\theta}(i,u) = -k_{iu}(V^\theta)\gamma^1_{iu}(\theta).
\end{equation}
\begin{lemma}\label{lem2}
$\displaystyle \sum_{i\in S^-} V^\theta(i)$ is a Lyapunov function of \eqref{actor_ode}.
\end{lemma}
\begin{proof}
Let $Pr(i\rightarrow s,k,\pi)$ be the probability of transitioning from state i to state s in k steps following policy $\pi$. Let $Q^\theta$ as in \eqref{Q-def}.
From chapter 13 of \cite{SuttonBarto2018},
\[
\begin{split}
&\nabla_\theta \sum_{i\in S^-} V^\theta(i)=\sum_{i \in S^-}\sum_{s\in S^-}\sum_{k=0}^\infty Pr(i\rightarrow s,k,\pi)\\
&\sum_{u \in A} \nabla_\theta\pi_\theta(s,u)Q^\theta(s,u)\\
&=\Big[\sum_{i \in S^-}\sum_{k=0}^\infty Pr(i\rightarrow s,k,\pi)\pi_\theta(s,u)k_{su}(V^\theta),\\
&s\in S^-,u\in A\Big]^T.
\end{split}
\]
The last equality is by Assumption~\ref{a2}. Hence:
\[
\frac{d \displaystyle \sum_{i\in S^-}V^\theta(i)}{dt}=\Big(\nabla_\theta\sum_{i\in S^-} V^\theta(i)\Big)^T\dot{\theta}\leq 0.
\]
with equality only at equilibrium points of \eqref{actor_ode}.
\end{proof}
By Lemma~\ref{lem1} and Lemma~\ref{lem2}, (A3) of \cite{konda1999} is verified.
Let $\hat{\theta}$ be a stable equilibrium point of \eqref{actor_ode}.
\begin{lemma}\label{lem-main}
For $\epsilon>0$ there exists $\bar{\theta}$ such that for all $\theta^0\geq\bar{\theta}$ we have $V^*(i)-\epsilon\leq  V^{\hat{\theta}}(i)\leq V^*(i)+\epsilon$, $\forall i \in S^-$.
\end{lemma}
\begin{proof}
The proof is adapted from Lemma 5.12 of \cite{konda1999}. For some $V$ as in \eqref{V-def}:
\[
\begin{split}
||V-V^*||_\xi&\leq ||V-TV||_\xi+||TV-V^*||_\xi\\
&=||V-TV||_\xi+||TV-TV^*||_\xi\\
&\leq ||V-TV||_\xi+\beta||V-V^*||_\xi.
\end{split}
\]
Hence $||V-V^*||_\xi\leq (1-\beta)^{-1}||TV-V||_\xi$. For $i \in S^-, u\in A$ let us define:
\[
k_{iu}^-(V)=\max(0,-k_{iu}(V)),k_{iu}^+(V)=\max(0,k_{iu}(V)).
\]
Now,
\[
\begin{split}
||TV^{\hat{\theta}}-V^{\hat{\theta}}||_\xi&=\max_{i \in S^-}\frac{\displaystyle|\min_{u \in A}k_{iu}(V^{\hat{\theta}})|}{\xi(i)}\\
&\leq \frac{\displaystyle\max_{i \in S^-,u \in A}k_{iu}^-(V^{\hat{\theta}})}{\xi_{min}}
\end{split}
\]
where $\displaystyle\xi_{min}=\min_{i\in S^-}\xi(i)$. It can be seen that:
\[
\pi_{\hat{\theta}}(i,u)k_{iu}(V^{\hat{\theta}})=\frac{\exp(-\theta^0) k_{iu}^+(V^{\hat{\theta}})- \exp(\theta^0)k_{iu}^-(V^{\hat{\theta}})}{\displaystyle\sum_{u' \in A}\exp(\hat{\theta}(i,u'))}
\]
Summing over all $u \in A$, LHS becomes $0$, hence:
\[
\sum_{u\in A} k_{iu}^-(V^{\hat{\theta}})=\exp(-2\theta^0)\sum_{u\in A} k_{iu}^+(V^{\hat{\theta}})
\]
Let $|g(\cdot,\cdot,\cdot)|\leq B$ and $\displaystyle\xi_{max}=\max_{i\in S^-}\xi(i)$,  
then $\displaystyle\max_{i \in S^-}|V^{\hat{\theta}}(i)|\leq B\xi_{max}$. Then it can be seen that 
\[
|k_{iu}(V^{\hat{\theta}})|\leq B(1+2\xi_{max}).
\]
Let $|A|$ be the cardinality of A. Then,
\[
\begin{split}
&||V-V^*||_\xi\leq \frac{B(1+2\xi_{max})|A|}{\xi_{min}\exp(2\theta^0)(1-\beta)}\\
\Rightarrow &\max_{i \in S^-}|V(i)-V^*(i)|\leq \frac{\xi_{max}B(1+2\xi_{max})|A|}{\xi_{min}\exp(2\theta^0)(1-\beta)}
\end{split}
\]
The claim follows by choosing $\theta^0$ such that the RHS $\leq \epsilon$.
\end{proof}
See Theorem~\ref{main-acca} for main result.
\subsection{Convergence of Critic Actor}
Since our critic is updated on the slower timescale, we let $V_n \approx V$ for the analysis of the faster recursion with $V$ as in \eqref{V-def}. 

Consider now the ODE: For $i \in S^-,u \in A$,
\begin{equation}
	\label{ode-a}
	\dot{\theta}(i,u) = \bar{\Gamma}_{\theta^0}(\theta(i,u);-k_{iu}(V)).
\end{equation}
Let,
\[
\gamma^2_{iu}(\theta) =\left\{
\begin{array}{ll}
0 & \mbox{ if } \theta(i,u) = \theta^0 \mbox{ and } k_{iu}(V) \leq 0,\\
   & \mbox{ or } \theta(i,u) = -\theta^0 \mbox{ and } k_{iu}(V) \geq 0,\\
1  & \mbox { otherwise.}
\end{array}
\right.
\]

Thus, the ODE \eqref{ode-a} takes the form:
\begin{equation}
	\label{ode-b}
	\dot{\theta}(i,u) = -k_{iu}(V)\gamma^2_{iu}(\theta).
\end{equation}

\begin{lemma}
\label{lem3}
The stable attractors of \eqref{ode-b} is the set $\Theta(V)$ consisting of $\theta_V$ that satisfy $\theta_V(i,u) =\theta^0$ for $k_{iu}(V)< 0$,  $\theta_V(i,u)=-\theta^0$ for  $k_{iu}(V)>0$, and $\theta_V(i,u) \in [-\theta^0,\theta^0]$ for $k_{iu}(V)=0$.
\end{lemma}
\begin{proof}
Easy to see from the dynamics of the ODE. 
\end{proof}
\begin{lemma}\label{lem4}
We have, $\theta_n \to \Theta(V)$ a.s.
\end{lemma}
\begin{proof}
Follows using same arguments as Lemma 4.9 of \cite{konda1999}.
\end{proof}
The value update \eqref{critic2} in this algorithm is on the slower time scale and sees the policy update \eqref{actor2} as quasi-equilibrated. Therefore we rewrite \eqref{critic2} as the following Stochastic Recursive Inclusion (SRI) (see Chapter 5 of \cite{borkar2022}): For $i \in S^-$,
\begin{equation}
\begin{split}
&V_{n+1}(i) = V_n(i) + \bar{a}(n)I\{Y_n=i\}\\
&\big[\sum_{u \in A}\pi_{\theta_{V_n}}(i,u)k_{iu}(V_n)+N_{n+1}(i)+\varepsilon_n(i)\big],
\end{split}
\end{equation}
where $\theta_{V_n} \in \Theta(V_n)$. By Lemma~\ref{lem4}, $\varepsilon_n(i) \rightarrow 0$.

We now verify the conditions 1.1 and 1.3 of \cite{benaim2005} to get the limiting Differential Inclusion (DI). Note that the iterates are stable, i.e., $\sup_n ||V_n||<\infty$ a.s.~ and the noise converges by Lemma 4.5 of \cite{konda1999}, hence condition 1.3 of \cite{benaim2005} is verified. 

For $i \in S^-$, let us define:
\[
\bar{K}_i(V)=\Big\{\sum_{u\in A}\pi_{\theta_{V}}(i,u)k_{iu}(V)|\theta_V\in \Theta(V)\Big\}.
\]
It can be seen that,
\[
\begin{split}
&\bar{K}_i(V)=\Bigg\{\sum_{u\in A}\Big(\frac{\exp(-\theta^0)+\exp(\theta^0)}{2}k_{iu}(V)\\
&+\frac{\exp(-\theta^0)-\exp(\theta^0)}{2}|k_{iu}(V)|\Big)\Big/\sum_{u' \in A}\exp(\theta_V(i,u'))\\
&\Bigg|\theta_V\in \Theta(V)\Bigg\}.
\end{split}
\]
As $\Theta(V)$ is a hyper-rectangle, it can be seen from the above that $\bar{K}_i(V)$ is a closed interval, hence convex and compact. Also it is easy to see that
\[
\sup_{\bar{k}_V \in \bar{K}_i(V)}||\bar{k}_V||\leq C_3(1+||V||),
\] 
for some $C_3>0$. Let $V_n \rightarrow V$ then $k_{iu}(V_n) \rightarrow k_{iu}(V)$ for $i \in S^-,u\in A$. If $\theta_n \rightarrow \theta$ and $\theta_n \in \Theta(V_n)$ then by taking possible values of the sequences $\{\theta_n\}$ and $\{k_{iu}(V_n)\}$, we can see that $\theta \in \Theta(V)$. Now by continuity of $\pi_\theta$, $\bar{K}_i$ is upper semi-continuous. Hence condition 1.1 of \cite{benaim2005} is verified. Let us consider the DI: For $i \in S^-$,
\begin{equation}\label{ca_critic_di1}
\dot{V}(i)\in \bar{K}_i(V).
\end{equation}
Let us define the linearly interpolated trajectory $\bar{V}(t),t \geq 0$ of $V_n$, see equation (IV) of \cite{benaim2005}. Therefore by Theorem 4.2 of \cite{benaim2005} and the arguments of Lemma 4.10 of \cite{konda1999}, the limit points of $\bar{V}(t+\cdot),t \geq 0$ satisfy the DI \eqref{ca_critic_di1}. Let $k_u(V)=[k_{iu}(V),i\in S]^T$ and $|k_u(V)|=[|k_{iu}(V)|,i\in S]^T$ for $u \in A$. Let $P_u,u \in A$ be the matrix with elements:
\[
P_u(i,j)=p(i,u,j),i,j\in S^-.
\]
Let $e_1:=\frac{\exp(-\theta^0)+\exp(\theta^0)}{2}$ and $e_2:=\frac{\exp(-\theta^0)-\exp(\theta^0)}{2}$. Now let,
\[
\begin{split}
&\mathcal{L}(V)=\frac{1}{2}\sum_{u \in A} (e_1k_u(V)^T(I-P_u)^{-1}k_u(V)\\
&+e_2k_u(V)^T(I-P_u)^{-1}|k_u(V)|).
\end{split}
\]
\begin{lemma}
$\mathcal{L}(V)$ is a Lyapunov function for \eqref{ca_critic_di1}.
\end{lemma}
\begin{proof}
Let $\bar{K}(V)$ be defined as the Cartesian product of $\bar{K}_i(V),i \in S^-$:
\[
\bar{K}(V)=\bigtimes_{i\in S^-}\bar{K}_i(V).
\]
Now,
\[
\begin{split}
&\sup_{\bar{k}_V\in \bar{K}(V)}\nabla_V \mathcal{L}(V)^T\bar{k}_V\\
&=\sup_{\theta_V\in \Theta(V)}-\sum_{i\in S^-}\sum_{u\in A}(e_1k_{iu}(V)^T+e_2|k_{iu}(V)|^T)\\
&\frac{\displaystyle\sum_{u\in A}(e_1k_{iu}(V)+e_2|k_{iu}(V)|)}{\displaystyle\sum_{u'\in A}\exp(\theta_V(i,u'))}\leq 0.
\end{split}
\]
with equality only at equilibrium points of \eqref{ca_critic_di1}.
\end{proof}
For stability of DIs the reader is referred to section 15 of \cite{filippov1988} and chapter 6 of \cite{aubin1984}. Let $\hat{V}$ be a stable equilibrium of \eqref{ca_critic_di1}.
\begin{lemma}\label{lem-CA}
We have, $(V_n,\theta_n) \to (\hat{V},\Theta(\hat{V}))$ a.s.
\end{lemma}
\begin{proof}
The claim follows from Proposition 3.27 of \cite{benaim2005}.
\end{proof}
\begin{lemma}\label{lem-main2}
For $\epsilon>0$ there exists $\bar{\theta}$ such that for all $\theta^0\geq\bar{\theta}$ we have $V^*(i)-\epsilon\leq  \hat{V}(i)\leq V^*(i)+\epsilon$, $\forall i \in S^-$.
\end{lemma}
\begin{proof}
Same as Lemma~\ref{lem-main} by replacing $V^{\hat{\theta}}$ with $\hat{V}$.
\end{proof}

\subsection{Main Result}
Under Assumptions~\ref{a1}-\ref{a3}, we have the following main result:
\begin{theorem}\label{thm1}
	\label{main-acca}
	Given $\epsilon >0$, {$\exists \ \bar{\theta} \equiv \bar{\theta}(\epsilon) > 0$} such that for all $\theta^0 \geq \bar{\theta}$, $(V_n, \theta_n)$, $n\geq 0$, governed according to \eqref{critic2}-\eqref{actor2}, converges almost surely for both AC and CA, to the set
\[
\begin{split}
&\{(\hat{V},\hat{\theta}) \mid V^* (i)-\epsilon \leq \hat{V}(i) \leq  V^* (i) + \epsilon, \forall i \in S^-;\\
&\hat{\theta}\in \Theta(\hat{V})\}.
\end{split}
\]
\end{theorem}
\begin{proof}
For AC, we can see from above that assumptions (A1)-(A6) of \cite{konda1999} are verified and the claim follows from Lemma 4.10 of \cite{konda1999} and Lemma~\ref{lem-main}. For CA, the claim follows from Lemmas~\ref{lem-CA} and ~\ref{lem-main2} above.
\end{proof}

\section{An Actor-Critic Algorithm with Function Approximation}
\begin{algorithm}[tb]\label{algorithm2}
\caption{Actor-Critic with Function Approximation}
\textbf{Input}: Step-sizes satisfying Assumption~\ref{a4}.
\begin{algorithmic}[1]
\STATE Initialize $v_0;\theta_0$.
\FOR{each episode \(n \geq 0 \)}
\STATE Initialize starting state $i_0(n)$
\FOR{each timestep \(k \geq 0 \)}
\STATE Play action $u_k(n)$ from state $i_k(n)$ and collect next state $i_{k+1}(n)$ and cost $g(i_k(n),u_k(n),i_{k+1}(n))$.
\IF{$i_{k+1}(n)=i_0$}
\STATE break
\ENDIF
\ENDFOR
\STATE Perform critic update \eqref{critic1}.
\STATE Perform actor update \eqref{actor1}.
\ENDFOR
\end{algorithmic}
\end{algorithm}
We propose an online Actor Critic (AC) algorithm with function approximation for the SSP problem. For our function approximation algorithm we consider Assumption~\ref{a1}. Our step-sizes are assumed to satisfy the following requirements:
\begin{assumption}[Step-Sizes]
\label{a4}
$a(n),b(n),n\geq 0$ are positive sequences satisfying the following conditions. Let $c(n),n\geq 0$ represent $a(n),b(n),n\geq 0$, then
\begin{itemize}
\item[(i)] ${\displaystyle \sum_n c(n)=\infty}$, ${\displaystyle \sum_n c(n)^2<\infty}$.
\end{itemize}
This is the standard Robbins-Monro condition for step-sizes.
\begin{itemize}
\item[(ii)] $\frac{b(n)}{a(n)}\rightarrow 0$ as $n\rightarrow\infty$.
\end{itemize}
This is the standard two-timescale stochastic approximation condition, see Chapter 8 of \cite{borkar2022}.
\end{assumption}
\begin{remark}
 The following examples satisfy the above conditions: $a(n)=\log(n+2)/(n+2),b(n)=1/(n+1)$; $a(n)=1/(n+1)^{\alpha_1},b(n)=1/(n+1)^{\alpha_2}$, $0.5<\alpha_1,\alpha_2\leq 1$ with $\alpha_1<\alpha_2$. To increase or decrease the step sizes we can multiply any positive constant to them which will also satisfy these conditions.
\end{remark}
We also have the following assumption.
\begin{assumption}\label{a5}
Under every policy, there is a positive probability of visiting every state starting from any state.
\end{assumption}
We approximate the value function $V^\pi(i),i \in S^-$ using linear function approximation $V^\pi(i) \approx {v^\pi}^T\phi(i)$ where $v^\pi \in \mathbb{R}^{d_1}$ is the parameter vector and $\phi(i)$ is the feature vector for state $i$. We take $\phi(i_0)=0 \in \mathbb{R}^{d_1}$. Let $\Phi \in \mathbb{R}^{|S^-|\times d_1}$ be the feature matrix with $\phi(i)^T,i \in S^-$ as the rows. We have the following assumption.
\begin{assumption}\label{a6}
The columns of the matrix $\Phi$ are linearly independent. Also $d_1<<|S^-|$.
\end{assumption}

We parameterize the policy $\pi$ with parameter $\theta \in \mathbb{R}^{d_2}$ denoted by $\pi_\theta$. We will use $\pi$ and $\theta$ interchangeably henceforth. We have the following assumption.

\begin{assumption}\label{a7}
$\pi_\theta$ is twice continuously differentiable in $\theta$. Further $\pi_\theta(i,u)>0, \forall i \in S^-,u \in A$.
\end{assumption}
Let us define $\psi_\theta(i,u)=\nabla_\theta \log \pi(i,u)$. Let $\Gamma^1$ be a projection operator that projects the iterates to some convex and compact set $\Theta^P$. An episode comprises of a sequence of states, actions, rewards, and next states, starting from the start state and ending at  the terminal state $i_0$. For each episode $n\geq 0$, let the $k$-th state be denoted $i_k(n)$, the $k$-th action taken be denoted $u_k(n)$ and the length of the episode be denoted $\tau'(n)$. The action $u_k(n)$ is sampled from $\pi_{\theta_n}(i_k(n),\cdot)$. 
\[
\begin{split}
&d_k(n)=g(i_k(n),u_k(n),i_{k+1}(n))\\
&+v_n^T\phi(i_{k+1}(n))-v_n^T\phi(i_k(n)),
\end{split}
\]
\begin{equation}\label{critic1}
\begin{split}
&v_{n+1}=v_n+a(n)\sum_{k=0}^{\tau'(n)-1}d_k(n)\phi(i_k(n)),
\end{split}
\end{equation}
\begin{equation}\label{actor1}
\begin{split}
&\theta_{n+1}=\Gamma^1[\theta_n-b(n)\sum_{k=0}^{\tau'(n)-1}d_k(n)\psi_{\theta_n}(i_k(n),u_k(n))].
\end{split}
\end{equation}
Here $d_k(n)$ is the temporal difference (TD) term and it appears in both the actor and critic recursions. Our critic update \eqref{critic1} is taken from Chapter 6 of \cite{NDPbook}. Our actor update is corresponding to the policy gradient given in Chapter 13 of \cite{SuttonBarto2018}. The novelty of our algorithm is that it is a two-timescale algorithm and possesses an asymptotic almost-sure convergence guarantee, as seen in Section~\ref{converence-fa}.
\begin{remark}
The summations in recursions \eqref{critic1} and \eqref{actor1} can be performed recursively and there is no need to store all the data of a particular episode. For a maximization problem, change the $-$ sign to $+$ sign after $\theta_n$ in recursion \eqref{actor1}.
\end{remark}
The algorithm is formally given in Algorithm 2.

\section{Convergence Analysis of Algorithm in Function Approximation Setting}
\label{converence-fa}

We proceed to prove the convergence of our Actor-Critic (AC) in the function approximation setting. Chapter 8 of \cite{borkar2022} provides the theory of Two Timescale Stochastic Approximation. We will use it to prove the convergence of our algorithm. The step-size conditions of Chapter 8 of \cite{borkar2022} are built in Assumption~\ref{a4}. Let
\[
\begin{split}
&\mathcal{F}^1_n=\sigma(v_m,\theta_m,m \leq n;i_k(m),u_k(m),0\leq k \leq \tau'(m)-1,\\
&m< n),n \geq 0
\end{split}
\]
be an increasing family of associated sigma fields. This keeps record of the history of the algorithm. Let us define the policy probability simplex:
\begin{equation}\label{simplex-def}
\begin{split}
&\Delta=\{\pi|0\leq \pi(i,u)\leq 1,\displaystyle\sum_{u' \in A} \pi(i,u')=1,\\
&\forall i \in S^-,u \in A\}.
\end{split}
\end{equation}
Let $R_\theta$ be as in \eqref{R-def} and $P_\theta$ as in \eqref{P-def}. Let $[h^\theta_0(i):=P(X_0=i),i \in S^-]^T$ be the initial state distribution. Then the state distribution at time instant $n$ will be ${h^\theta_n}^T={h^\theta_0}^T P_\theta^n$. Let us define the vector:
\[
h^\theta=\lim_{n \to \infty}\sum_{k=0}^n h_k^\theta
\]
By Lemma~\ref{lem1}, we know that the eigenvalues of $P_\theta$ are $<1$. So the above limit exists and  $q^\theta={q^\theta_0}^T(I-P_\theta)^{-1}$.
Let $H^\theta_n,n \geq 0$ be the matrix with elements $h^\theta_n$ along the diagonal and $H^\theta$ be the matrix  with elements $h^\theta$ along the diagonal.
Let us define $A^1(\theta),b^1(\theta)$ as:
\[
\begin{split}
A^1(\theta)=&E[\sum_{k=0}^{\tau'(n)-1} \phi(i_k(n))(\phi(i_{k+1}(n))-\phi(i_k(n)))^T]\\
=&\Phi^T \sum_{k=0}^\infty H_k^\theta (P_\theta-I)\Phi=\Phi^T H^\theta(P_\theta-I)\Phi,
\end{split}
\]
\[
\begin{split}
b^1(\theta)=&E[\sum_{k=0}^{\tau'(n)-1} \phi(i_k(n))g(i_k(n),u_k(n),i_{k+1}(n))]\\
=&\Phi^T \sum_{k=0}^\infty H_k^\theta R_\theta=\Phi^T H^\theta R_\theta.
\end{split}
\]
Then the recursion \eqref{critic1} can be written as:
\[
v_{n+1}=v_n+a(n)(f^1(\theta_n,v_n)+M^1_{n+1}),
\]
where,
\[
f^1(\theta,v)=A^1(\theta)v+b^1(\theta)
\]
and $M^1_{n+1}$ is the noise term defined as:
\[
M^1_{n+1}=\sum_{k=0}^{\tau'(n)-1}d_k(n)\phi(i_k(n))-A^1(\theta_n)v_n-b^1(\theta_n).
\]
Also, the recursion \eqref{actor1} can be written as:
\[
\theta_{n+1}=\Gamma^1[\theta_n+b(n)(f^2(\theta_n,v_n)+M^2_{n+1})],
\]
where
\[
\begin{split}
&f^2(\theta,v)=\sum_{i \in S^-} h^\pi(i)\sum_{u \in A} \pi(i,u)\Big(\sum_{j \in S}p(i,u,j)g(i,u,j)\\
&\sum_{j\in S^-}p(i,u,j)v^T\phi(j)-v^T\phi(i)\Big)\psi_\theta(i,u),
\end{split}
\]
and $M^2_{n+1}$ is the noise term defined as:
\[
M^2_{n+1}=\sum_{k=0}^{\tau'(n)-1}d_k(n)\psi_{\theta_n}(i_k(n),u_k(n))-f^2(\theta_n,v_n)
\]
We have the following noise lemmas.
\begin{lemma}\label{critic-noise-condition}
For some $C_4>0$,
\[
E[||M_{n+1}^1||^2|\mathcal{F}^1_n]\leq C_4(1+||v_n||^2).
\]
\end{lemma}
\begin{proof}
Let $|g(\cdot,\cdot,\cdot)|\leq B,||\phi(.)||\leq B_1$. Let $e$ be the vector of all ones. Then it can be seen that $E[\tau'(n)|\mathcal{F}^1_n]={h^{\theta_n}_0}^T(I-P_{\theta_n})^{-1}e\leq \max_{\pi \in \Delta}{h^{\pi}_0}^T(I-P_{\pi})^{-1}e$, which is uniformly bounded as $\Delta$ is compact. Now using eq (5) of \cite{sobel1982} and taking discount factor=$1$ and reward=$1$ it can be seen that:
\[
E[\tau'(n)^2|\mathcal{F}^1_n]={h^{\theta_n}_0}^T(I-P_{\theta_n})^{-1}(I+2P_{\theta_n}(I-P_{\theta_n})^{-1})e,
\]
which can be seen to be uniformly bounded as above. Then,
\[
\begin{split}
&||M^1_{n+1}||^2\leq \Big(\sum_{k=0}^{\tau'(n)-1}(B+2B_1||v_n||)B_1\\
&+||A^1(\theta_n)||||v_n||+||b^1(\theta_n)||\Big)^2\\
&=(BB_1\tau'(n)+||b^1(\theta_n)||+(2B_1^2\tau'(n)\\
&+||A^1(\theta_n)||)||v_n||)^2.
\end{split}
\]
Taking $C_4'=\max(BB_1\tau'(n)+||b^1(\theta_n)||,2B_1^2\tau'(n)+||A^1(\theta_n)||)$,
\[
||M^1_{n+1}||^2\leq C_4'^2(1+||v_n||)^2.
\]
The claim follows by noting that $(1+a)^2\leq 2(1+a^2)$ for $a \in \mathbb{R}$.
\end{proof}
\begin{lemma}\label{actor-noise-condition}
For some $C_6>0$,
\[
E[||M_{n+1}^2||^2|\mathcal{F}^1_n]\leq C_6(1+||v_n||^2).
\]
\end{lemma}
\begin{proof}
By Assumption~\ref{a7} and the fact that $\theta_n$ lies in a compact space, let $||\psi_{\theta_n}(.,.)||\leq B_2$. The claim follows from similar arguments as in Lemma~\ref{critic-noise-condition}.
\end{proof}

Hence the noise conditions of Chapter 8 of \cite{borkar2022} are satisfied. We will use the results of \cite{ramaswamy2017} to prove the stability of our critic updates \eqref{critic1}.
\begin{lemma}\label{fa-stabilty-lemma}
$\sup_n ||v_n||<\infty$ a.s.
\end{lemma}
\begin{proof}
Let $\Delta$ be as in \eqref{simplex-def}. We can write \eqref{critic1} as the following Stochastic Recursive Inclusion (SRI):
\begin{equation}\label{critic-sri}
\begin{split}
&v_{n+1}=v_n+a(n)(A^1(\pi)v_n+b^1(\pi)+M^1_{n+1}),
\end{split}
\end{equation}
where $\pi \in \Delta$. We will verify (A1)-(A5) of \cite{ramaswamy2017} to conclude. For $v \in \mathbb{R}^{d_1}$, let
\begin{equation}\label{f1-def}
f^3(v)=\{A^1(\pi)v+b^1(\pi)|\pi \in \Delta\}.
\end{equation}
Let $f^3(v)_p,p=0,\dots,d_1-1$, be the $p$th component of $f^3(v)$ which means $f^3(v)=\displaystyle \bigtimes_{p=0,\dots,d_1-1} f^3(v)_p$. Then $f^3(v)_p$ can be written as:
\[
\begin{split}
&f^3(v)_p=\Big\{ \sum_{i \in S^-} h^\pi(i)\sum_{u \in A} \pi(i,u)\Big(\sum_{j \in S}p(i,u,j)g(i,u,j)\\
&\sum_{j\in S^-}p(i,u,j)v^T\phi(j)-v^T\phi(i)\Big)\phi(i)_p\Big|\pi \in \Delta\Big\},
\end{split}
\]
where $\phi(i)_p$ is the $p$-th component of $\phi(i)$. Also, $f^3(v)_p$ is a continuous function of $\pi$ and since $\Delta$ is compact, $f^3(v)_p$ is a closed interval. Hence for any $v \in \mathbb{R}^{d_1}$, $f^3(v)$ is convex and compact. For any $v \in \mathbb{R}^{d_1}$, it is easy to see that $\sup_{y\in f^3(v)}||y||\leq C_5(1+||v||)$ for some $C_5>0$. $A^1(\pi)v+b^1(\pi)$ is a continuous function of $(\pi,v)$, hence it follows that $f^3(v)$ is upper semi-continuous. (A1) of \cite{ramaswamy2017} is verified. (A2) of \cite{ramaswamy2017} is satisfied by our step-size conditions. By Lemma~\ref{critic-noise-condition}, (A3) of \cite{ramaswamy2017} is verified. Let us now verify (A4) of \cite{ramaswamy2017}. For $c\geq 1$, let
\[
\begin{split}
&f^3_c(v)=\Big\{\frac{b^1(\pi)+A^1(\pi)cv}{c}, \pi \in \Delta \Big\},\\
&f^3_\infty(v)=\lim_{c \to \infty} f_c(v)=\{A^1(\pi)v|\pi \in \Delta\}.
\end{split}
\]
It can be seen from Lemma 6.10 of \cite{NDPbook} and the following argument that $A^1(\pi)$ is negative definite for any $\pi \in \Delta$. Hence,  $W(v)= v^Tv$ serves as a Lyapunov function for the DI $\dot{v}(t) \in f^3_\infty(v(t))$ with $0 \in \mathbb{R}^{d_1}$ as its global attractor. Hence, (A4) of \cite{ramaswamy2017} is verified. Next, let $v_n \to v$, $y_n \to y$ and $c_n \uparrow \infty$ such that $y_n \in f^3_{c_n}(v_n)$. Now $f^3_c(v)$ is an upper semi-continuous function of $(1/c,v)$, hence,  $y \in f^3_\infty(v)$. Thus, (A5) of \cite{ramaswamy2017} is also verified. The claim now follows from Theorem 1 of \cite{ramaswamy2017}.
\end{proof}
Also $\{\theta_n\}$ obtained from \eqref{critic1} is bounded due to projection. Hence (A3) of Chapter 8 of \cite{borkar2022} is verified.
Let $\bar{\Gamma^1}$ be the directional derivative of $\Gamma^1$ as in \eqref{dir-def}. Then, the recursion \eqref{actor1} can be written as:
\[
\theta_{n+1}=\theta_n+b(n)[\bar{\Gamma^1}(\theta_n;f^2(\theta_n,v_n)+M^2_{n+1})+o(b(n))].
\]

By arguing as in Lemma 8.1 of \cite{borkar2022}, $(v_n,\theta_n)$ in \eqref{critic1}-\eqref{actor1} can be seen to track the ODEs:
\[
\dot{v}(t)=f^1(\theta(t),v(t)),\dot{\theta}(t)=0.
\]
By the above, $\theta(t)\equiv \theta$ (a time-invariant actor parameter). 
Let us then consider the ODE:
\begin{equation}\label{fa-ode-1}
\dot{v}(t)=f^1(\theta,v(t)).
\end{equation}
We have the following:
\begin{lemma}
\label{fa-lem1}
$v^\theta={A^1(\theta)}^{-1}b^1(\theta)$ is the unique globally asymptotically stable equilibrium of \eqref{fa-ode-1}. Also $v^\theta$ is a Lipschitz continuous function of $\theta$.
\end{lemma}
\begin{proof}
It can be seen from Lemma 6.10 of \cite{NDPbook} and the following argument that $A^1(\theta)$ is negative definite. The first claim follows.

By Assumption~\ref{a7} and the fact that $\theta$ lies in a compact set, it can be seen that $\nabla_\theta {A^1(\theta)}^{-1}$ and $\nabla_\theta b^1(\theta)$ are uniformly bounded. The second claim follows.
\end{proof}
Hence (A1) of Chapter 8 of \cite{borkar2022} is verified.
\begin{lemma}\label{fa-lem2}
We have that $v_n\to v^{\theta}$ a.s.
\end{lemma}
\begin{proof}
Follows from Lemma 8.1 of \cite{borkar2022}.
\end{proof}
\begin{remark}
It can be seen that $f^1(\theta,v)$ is not a Lipschitz continuous function but a locally Lipschitz function for $||v||\leq B'_1$, for $B'_1>0$. But since by Lemma~\ref{fa-lem1}, the ODE \eqref{fa-ode-1} has a globally asymptotically stable equilibrium, the ODE takes values in a bounded set and the arguments of Lemma 8.1 of \cite{borkar2022} still hold.
\end{remark}
Hence, the recursion \eqref{actor1} can be written as:
\[
\theta_{n+1}=\Gamma^1(\theta_n+b(n)(f^2(\theta_n,v^{\theta_n})+M^2_{n+1})+\varepsilon^1_n).
\]
By Lemma~\ref{fa-lem2}, $\varepsilon^1_n \to 0$, almost surely.
Let us consider the ODE:
\begin{equation}\label{fa-ode2}
\dot{\theta}=\bar{\Gamma^1}(\theta;f^2(\theta,v^\theta)).
\end{equation}
Let us define a linearly interpolated (continuous) trajectory $\bar{\theta}(t),t \geq 0$ of $\theta_n$ as follows (cf.~ Chapter 2 of \cite{borkar2022}): Let $t(n) = \sum_{m=0}^{n-1} b(m)$, $n\geq 1$, with $t(0)=0$. Then $\bar{\theta}(t(n))=\theta_n$, $\forall n\geq 0$, and for $t\in (t(n),t(n+1))$, $\bar{\theta}(t)$ is obtained via a continuous linear interpolation of $\bar{\theta}(t(n))=\theta_n$ and $\bar{\theta}(t(n+1))=\theta_{n+1}$, $\forall n$. 

\begin{lemma}\label{fa-lem3}
The limit points of $\bar{\theta}(t+\cdot)$, $t\geq 0$, satisfy the ODE \eqref{fa-ode2}.
\end{lemma}
\begin{proof}
By Assumption~\ref{a7} and the fact that $\theta$ lies in a compact set, $\nabla_\theta \pi_\theta(i,u)$ and $\nabla^2_\theta \pi_\theta(i,u)$ are uniformly bounded. Hence, it can be seen that $\nabla_\theta f^2(\theta.v^\theta)$ is uniformly bounded. So $f^2$ is a Lipschitz continuous function of $\theta$. By Lemma~\ref{fa-stabilty-lemma} and Lemma~\ref{actor-noise-condition}, $E[||M^2_{n+1}||^2|\mathcal{F}^1_n]<\infty$ a.s. Hence the martingale sequence  $\sum_{n=0}^{m} b(n) M^2_{n+1}$, $m\geq 0$, converges almost surely (see Appendix C of \cite{borkar2022}). The claim follows by Theorem 5.3.1 of \cite{kushner1978}.
\end{proof}

Let $Q^\theta$ be as in \eqref{Q-def}. From Chapter 13 of \cite{SuttonBarto2018}, the gradient of $J(\theta)$ in \eqref{avg} is given as:
\[
\nabla_\theta J(\theta)=\sum_{i \in S^-} h^\theta(i) \sum_{u \in A}\nabla_\theta \pi_\theta(i,u) Q^\theta(i,u).
\]
Let us define the set:
\begin{equation}\label{Theta1-def}
\begin{split}
\Theta_1=\{\theta\in \Theta^P|\bar{\Gamma^1}(\theta;\nabla_\theta J(\theta))=0\},\\
\Theta_1^\epsilon=\{\theta \in \Theta^P|\inf_{\theta_1 \in \Theta_1}||\theta-\theta_1||\leq\epsilon\},
\end{split}
\end{equation}
where $\Theta^\epsilon_1$ is the $\epsilon$-neighborhood of $\Theta_1$. Let us define the approximation error:
\[
e'_1(\theta)=f^2(\theta,v^\theta)-\nabla_\theta J(\theta).
\]
\begin{remark}
The above error can be reduced if one selects good feature vectors that satisfy Assumption~\ref{a6}.
\end{remark}
\begin{lemma}\label{fa-lem4}
For $\epsilon>0$, $\exists \delta>0$ such that if $\sup_\theta||e'_1(\theta)||< \delta$, then $\theta_n \to \Theta^\epsilon_1$ almost surely.
\end{lemma}
\begin{proof}
Let us consider the ODE:
\begin{equation}\label{fa-ode3}
\dot{\theta}=\bar{\Gamma^1}(\theta;\nabla_\theta J(\theta)).
\end{equation}
Since $\nabla_\theta \pi_\theta(i,a)$ and $\nabla^2_\theta \pi_\theta(i,a)$ are uniformly bounded, it can be seen that $\nabla_\theta J(\theta)$ is Lipschitz continuous. Now using Gronwall's inequality, it can be seen that the limit points of $\bar{\theta}(t+\cdot),t \geq 0$ is a perturbed trajectory of \eqref{fa-ode3}. Now $\Theta_1$ is the set of stable equilibrium points of \eqref{fa-ode3}. The claim follows from Theorem 1 of \cite{hirsch1989}.
\end{proof}
Under Assumption~\ref{a1} and Assumptions~\ref{a4}-\ref{a7}, we have the following main theorem.
\begin{theorem}
Given $\epsilon>0$, $\exists \delta>0$ such that if $\sup_\theta||e'_1(\theta)||< \delta$ then, the iterates $v_n,\theta_n,n\geq 0$ of our AC algorithm \eqref{critic1}-\eqref{actor1} converge almost surely to the set
\[
\{(v^{\theta^*},\theta^*)|\theta^*\in \Theta_1^\epsilon\}.
\]
\end{theorem}
\begin{proof}
The claim follows by Lemma~\ref{fa-lem2} and Lemma~\ref{fa-lem4}.
\end{proof}

\section{Numerical Experiments}\label{numerical}
\begin{figure*}
\centering
\includegraphics[width=1\textwidth]{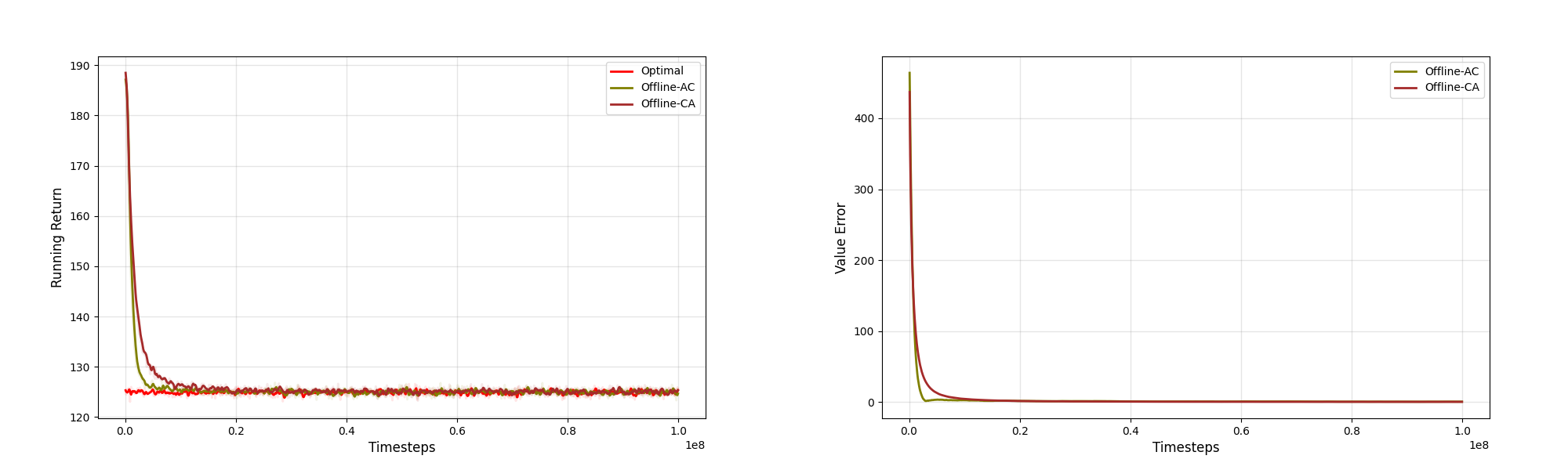}
\caption{Random MDP with State space size $|S|=20$, Action space size $|A|=4$}
\label{fig1}
\end{figure*}

\begin{figure*}
\centering
\includegraphics[width=1\textwidth]{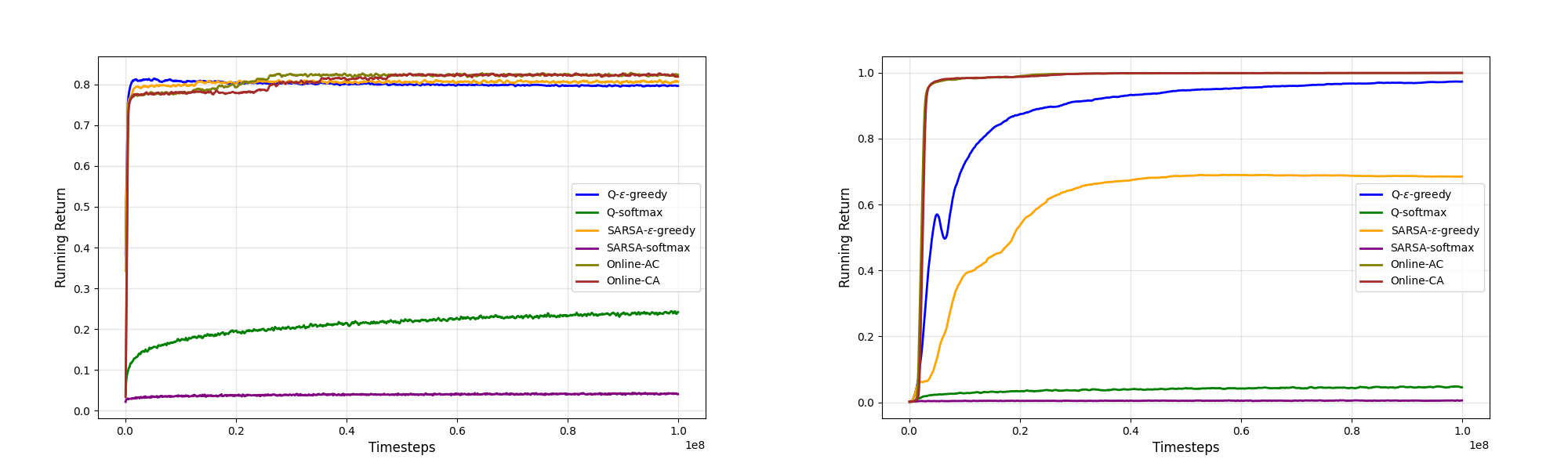}
\caption{FrozenLake gym Environemtnt with size 4x4(left) and 8x8(right)}
\label{fig2}
\end{figure*}

\begin{figure*}
\centering
\includegraphics[width=1\textwidth]{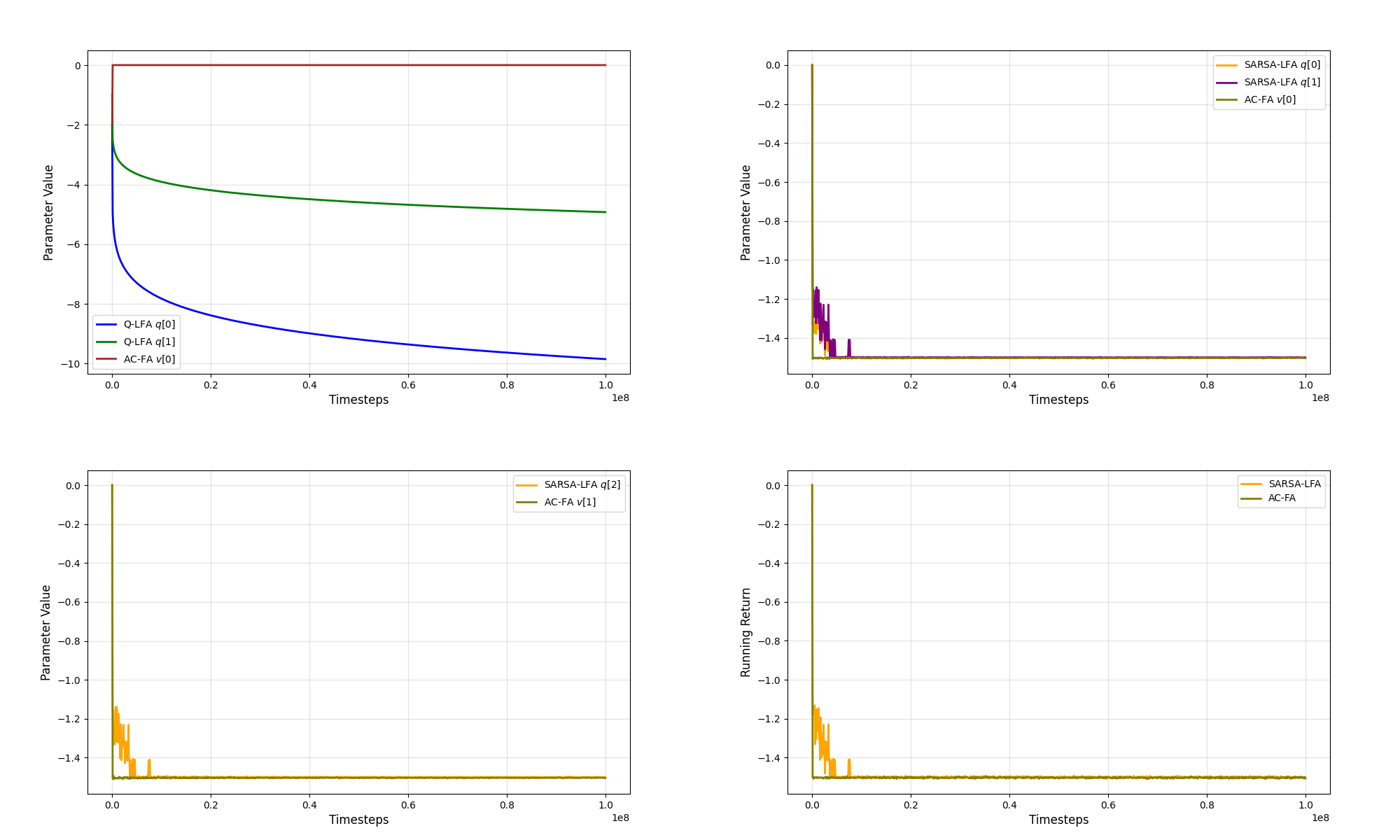}
\caption{Plots for comparison between AC-FA, Q-LFA and SARSA-LFA}
\label{fig3}
\end{figure*}

We perform experiments\footnote{All the code is available at  https://github.com/gsoumyajit/stochastic-shortest-path-ac} using our algorithms on various test beds. All the hyper-parameters used in the experiments have been optimized, see our implementation to check the actual hyper-parameters used in the experiments. All plots are averaged over 5 independent runs. As a verification that our tabular algorithms indeed converge to the optimal value function, we test them on a randomly generated MDP. Figure~\ref{fig1} shows experiments on a MDP with 20 states and 4 actions. The left plot of Figure~\ref{fig1} is for the running return. All the plots for running return in this section correspond to the average running return of the last $10^4$ episodes. The right plot in Figure~\ref{fig1} is for value error, which is calculated by the Euclidean distance between $V_n$ and $V^*$. We can see both the algorithms learn the optimal return and the corresponding value error goes to zero in either case. 

We next compare the performance of the online versions of our tabular algorithms AC and CA with Q-Learning and SARSA. The plots are given in Figure~\ref{fig2}. We perform these experiments on the FrozenLake gym environment with both $4\times 4$ and $8\times 8$ grid sizes. We test Q-Learning and SARSA with two exploration strategies: $\epsilon$-greedy with a decaying $\epsilon$-parameter and soft-max policy as proposed in \cite{singh2000convergence}. They are proposed as GLIE policies in \cite{singh2000convergence}, that ensure convergence of SARSA and also that all state-action pairs are sampled infinitely often in Q-Learning with proper exploitation. There are other types of exploration like UCB sampling \cite{wang2020} and Thompson sampling \cite{osband2013more}, but they need to be designed specifically for the SSP problem. We can see that although softmax exploration has convergence guarantees \cite{singh2000convergence}, it is extremely slow for both Q-Learning and SARSA in both $4\times 4$ and $8\times 8$ environments. Whereas $\epsilon$-greedy exploration with decaying $\epsilon$ has convergence guarantees \cite{singh2000convergence}, it is very slow for both Q-learning and SARSA in the $8\times 8$ environment. We use the decay rates for $\epsilon$-greedy and softmax exploration as suggested in \cite{singh2000convergence}. There is an improvement of softmax exploration via the mellowmax operator \cite{asadi2017}, but the calculation of the decay rate is computationally expensive, hence we do not compare that with our tabular algorithms. We can see that both our Online AC and CA algorithms perform better than Q-Learning and SARSA with both of the aforementioned types of exploration.

We next test our AC algorithm with function approximation (AC-FA) with Q-Learning with linear function approximation (Q-LFA). The latter does not have any convergence guarantees for the SSP problem, in fact, it is known to diverge \cite{baird1995}, due to its off-policy nature \cite{tsitsiklis1997}. The example used in \cite{baird1995} is for the SSP problem, but it uses feature vectors of size greater than the size of the state space. This violates the linear independence assumption on basis functions, which is essential for convergence, see \cite{melo2008}. The off-policy divergence example of TD with LFA presented in \cite{tsitsiklis1997} is for the discounted cost problem. We combine both the ideas of \cite{baird1995} and \cite{tsitsiklis1997}, and come up with a SSP example where Q-LFA satisfies the condition of linear independence of basis functions, but it still diverges. The example has 2 non-terminal states: $i_1,i_2$, a terminal state: $i_0$ and 2 actions: $u_1,u_2$. The MDP diagram is given in Figure~\ref{fig4}, with probabilities written on the edges.
\begin{figure}
\centering
\begin{tikzpicture}[auto,node distance=5cm,>=latex,font=\large]

    \tikzstyle{round}=[thick,draw=black,circle]
    \tikzstyle{dot}=[thick,fill=black,circle]
    \node[round] (s0) {$i_1$};
    \node[dot,above =0.5cm of s0] (1) {};
    \node[dot,below =0.5cm of s0] (2) {};
    \node[round,right =of s0] (s1) {$i_2$};
    \node[dot,above =0.5cm of s1] (3) {};
    \node[dot,below =0.5cm of s1] (4) {};
    \node[round,right =2cm of s0] (s2) {$i_0$};
    
    \path[->] (s0) edge node {$u_1$} (1);
    \path[->] (s0) edge node {$u_2$} (2);
    \path[->] (s1) edge node {$u_1$} (3);
    \path[->] (s1) edge node {$u_2$} (4);
    
    \path[->] (1) edge [bend left] node[above,pos=0.2] {$0.9$} (s1);
    \path[->] (1) edge [bend left] node[above,pos=0.6] {$0.1$} (s2);
    \path[->] (2) edge [bend right] node[below,pos=0.2] {$0.1$} (s2);
    \path[->] (2) edge [bend left=100] node[left,pos=0.5] {$0.9$} (s0);
    \path[->] (3) edge [bend left=100] node[right,pos=0.5] {$0.9$} (s1);
    \path[->] (3) edge [bend right] node[above,pos=0.2] {$0.1$} (s2);
    \path[->] (4) edge [bend left] node[below,pos=0.6] {$0.1$} (s2);
    \path[->] (4) edge [bend left] node[below,pos=0.2] {$0.9$} (s0);
\end{tikzpicture}
\caption{MDP used for Q-Learning with LFA}
\label{fig4}
\end{figure}
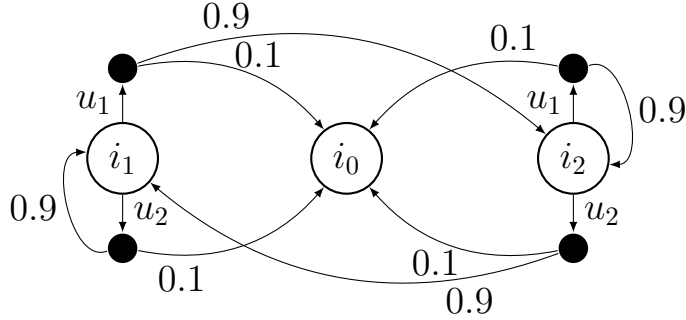
All costs are $0$. The initial state  distribution $h^\pi_0=[0.5,0.5]^T$. The features used for Q-LFA are: $\phi_1(i_1,u_1)=\phi_1(i_2,u_1)=[1,0]^T,\phi_1(i_1,u_2)=\phi_1(i_2,u_2)=[0,1]^T$. In our AC-FA algorithm we parameterize our policy with the softmax policy:
\[
\pi_\theta(i,u)=\frac{\exp(\theta^T\phi_1(i,u))}{\displaystyle\sum_{u' \in A} \exp(\theta^T\phi_1(i,u'))},i \in S^-, u \in A,\theta \in \mathbb{R}^2.
\]
We use the same $\phi_1$ features as above. The value features used for our AC-FA algorithm are $\phi(i_1)=\phi(i_2)=[1]^T$. We take the initial values of our AC-FA algorithm as $v_0=[-2.0]^T$ and $\theta_0=[-2.0,-1.0]^T$. We take the initial values of Q-LFA as $q_0=[-2.0,-1.0]^T$. The plots of the parameter values are given in the top left sub-figure of Figure~\ref{fig3}. We observe that the parameters of Q-LFA diverge to $-\infty$, whereas the value parameters of our AC-FA algorithm converge to $0$. We do not plot the policy parameters $\theta$, as they do not represent actual values of returns. For Q-LFA, we use constant exploration $\epsilon=1$. The step-sizes of Q-LFA satisfy the Robbins-Monro condition of Assumption~\ref{a4}(i). Note that the divergence rate of Q-LFA decreases due to the decreasing step-sizes.

Next, we test our AC-FA algorithm with SARSA with linear function approximation (SARSA-LFA). In \cite{melo2008}, it is proven that SARSA-LFA converges almost surely when the policy is Lipschitz continuous, with a small Lipschitz constant. But a  Lipschitz continuous policy that works with a small Lipschitz constant was not proposed. In \cite{zhang2023sarsa}, it is shown that the soft-max policy with high temperature works as a policy with small Lipschitz constant. It is also pointed out that higher temperature leads to poor exploitation while lower temperature leads to chattering. We test our AC-FA algorithm with SARSA-LFA with $\epsilon$-softmax policy as suggested in \cite{zhang2023sarsa}. We do not test our AC-FA algorithm with SARSA-LFA with $\epsilon$-greedy policy as suggested in \cite{gordon2000}, as it is not a theoretically rigorous procedure. We use the same MDP considered in \cite{zhang2023sarsa} that has three non-terminal states: $i_1,i_2,i_3$ and a terminal state: $i_0$. Two actions are feasible from $i_1$, namely $u_0,u_1$, while one action $u_0$ is feasible from both $i_2$ and $i_3$. The MDP diagram is given in Figure~\ref{fig5}, with action and rewards given on the edges.

\begin{figure}
\centering
\begin{tikzpicture}[auto,node distance=5cm,>=latex,font=\large]

    \tikzstyle{round}=[thick,draw=black,circle]
    \tikzstyle{dot}=[thick,fill=black,circle]
    \node[round] (s2) {$i_2$};
    \node[round,below =2cm of s2] (s3) {$i_3$};
    \node[round,right=3cm of $(s2)!0.5!(s3)$] (s0) {$i_0$};
    \node[round,left=3cm of $(s2)!0.5!(s3)$] (s1) {$i_1$};
    
    \path[->] (s1) edge node[above=0.1cm] {$u_0,0$} (s2);
    \path[->] (s1) edge node[below=0.1cm] {$u_1,0$} (s3);
    \path[->] (s2) edge node[above=0.1cm] {$u_0,-2$} (s0);
    \path[->] (s3) edge node[below=0.1cm] {$u_0,-1$} (s0);
\end{tikzpicture}
\caption{MDP used for SARSA with LFA}
\label{fig5}
\end{figure}
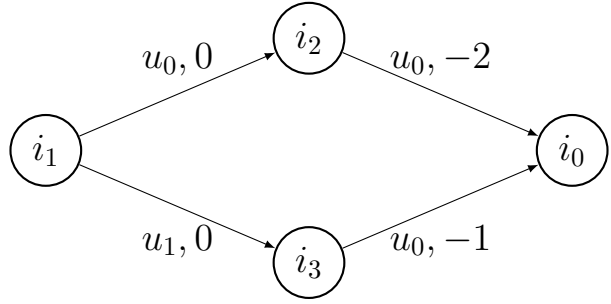
All the transition probabilities in this setting are $1$. The starting state is chosen to be $i_1$. The features used for SARSA-LFA are: $\phi_1(i_1,u_0)=[1,0,0]^T,\phi_1(i_1,u_1)=[0,1,0]^T,\phi_1(i_2,u_0)=\phi_1(i_3,u_0)=[0,0,1]^T$, as suggested by \cite{zhang2023sarsa}. For our AC-FA algorithm, we parameterize the policy with the following $\epsilon$-softmax policy:
\[
\begin{split}
&\pi_\theta(i,u)=\frac{\epsilon}{|A|}+(1-\epsilon)\frac{\exp(\theta^T\phi_1(i,u))}{\displaystyle\sum_{u' \in A} \exp(\theta^T\phi_1(i,u'))},\\
&i \in S^-, u \in A,\theta \in \mathbb{R}^3.
\end{split}
\]
We use the same $\phi_1$ features as above. We use $\epsilon=0.1$ for both SARSA-LFA and AC-FA. The value features used for our AC-FA algorithm are $\phi(i_1)=[1,0]^T,\phi(i_2)=\phi(i_3)=[0,1]^T$. All the parameters of SARSA-LFA and AC-FA start from zeros. The plots of parameter values are given in the top right and bottom left sub-figures of Figure~\ref{fig3}. The states $i_2$ and $i_3$ are given the same importance in both SARSA-LFA and AC-FA by the definition of the feature vectors. In this setting, the Q-parameters of SARSA-LFA and the value parameters of AC-FA are expected to converge to $-1.5$. We select the initial step-sizes for both SARSA-LFA and AC-FA to be $0.01$ and slowly decrease them to satisfy Assumption~\ref{a4}. In \cite{zhang2023sarsa} a constant step-size of $0.01$ is used for SARSA-LFA, and the Q-parameters are seen to chatter. It is however also claimed that with decreasing step-sizes the chattering happens though it slows down. However, in our experiments with decreasing step-sizes, we see that chattering disappears altogether after some time for SARSA-LFA.

We also see that chattering does not happen at all  for our AC-FA algorithm. In the bottom right sub-figure of Figure~\ref{fig3}, we plot the running return of SARSA-LFA and AC-FA. We can see similarly that chattering happens for SARSA-LFA but no chattering happens for AC-FA for the running return. We use temperature $\tau=0.01$ for SARSA-LFA, as suggested in \cite{zhang2023sarsa}, which means the Lipschitz constant of the behavioral policy is high and exploitation is high. Although the chattering disappears after sometime in our experiments, ideally for SARSA we should start from a high temperature (more exploration) and gradually decrease to low temperature (more exploitation). Currently, there is no theoretically backed decaying schedule for $\tau$ in the literature.

\section{Conclusions}

We presented two tabular algorithms: Actor-Critic (AC) and Critic-Actor (AC) for the Stochastic Shortest Path (SSP) problem. We also presented the Actor-Critic algorithm with function approximation (AC-FA) for the SSP problem. We show asymptotic almost sure convergence of all the algorithms. In Section~\ref{numerical}, we show that both our tabular AC and CA algorithms perform better than tabular Q-Learning and SARSA, when theoretically backed exploration and exploitation strategies are used. We also see that our AC-FA algorithm reliably converges for a diagnostic MDP, whereas Q-Learning with Linear function approximation (Q-LFA) diverges. We also see that our AC-FA algorithm converges reliably for a diagnostic MDP, whereas SARSA with linear function approximation (SARSA-LFA) chatters for sometime in the beginning. We also point out that SARSA-LFA does not have any theoretically backed decaying schedule for their exploration. To the best of our knowledge, our AC-FA is the only function approximation based algorithm that converges and performs reliably for the SSP problem.

In this paper, we developed our algorithms assuming that all policies are proper satisfying Assumption~\ref{a1}, which is also referred to as the Basic Assumption in Section~\ref{intro}. Our algorithms can be extended to the Standard Assumption setting in Section~\ref{intro}, by projecting the critic recursions to a large enough compact space. This will ensure that our critic estimates do not diverge for improper policies, and our actor will eventually learn the optimal proper policy. The projection technique will work for both our tabular and function approximation algorithms under the Standard Assumption.

As future work, one can explore constrained optimization algorithms for the SSP problem as done for the Discounted Cost Problem \cite{bhatnagar2010scl} and Average Cost Problem \cite{bhatnagar2012online}. One can also develop Natural Actor-Critic algorithms\cite{bhatnagar2009} for the SSP problem to improve the convergence behavior.
\bibliographystyle{IEEEtran}
\bibliography{IEEEfull}
\end{document}